\documentclass[11pt]{article}

\usepackage{amsmath,amssymb,amsfonts,amsthm}
\usepackage{mathrsfs}
\usepackage{bm}
\usepackage{graphicx}
\usepackage{booktabs}
\usepackage{enumitem}
\usepackage{natbib}
\usepackage{hyperref}
\usepackage{fullpage}
\usepackage{parskip}

\hypersetup{
  colorlinks=true,
  linkcolor=blue,
  citecolor=blue,
  urlcolor=blue
}

\newtheorem{theorem}{Theorem}[section]

\newtheorem{proposition}[theorem]{Proposition}
\newtheorem{corollary}[theorem]{Corollary}

\theoremstyle{definition}
\newtheorem{definition}[theorem]{Definition}

\newcommand{\E}{\mathbb{E}}
\newcommand{\Var}{\mathrm{Var}}

\newcommand{\R}{\mathbb{R}}

\newcommand{\sH}{\mathscr{H}}

\newcommand{\sW}{\mathscr{W}}

\DeclareMathOperator*{\argmin}{arg\,min}

\title{\textbf{A General Weighting Theory for Ensemble Learning}\\
\large Beyond Variance Reduction via Spectral and Geometric Structure}

\author{
{\bf Ernest Fokou\'e}\\[0.5em]
School of Mathematics and Statistics\\
Rochester Institute of Technology\\[0.3em]
{\tt eMail:epfeqa@rit.edu}
}

\date{}

\begin{document}
\maketitle

% ------------------------------------------------------------
% ABSTRACT
% ------------------------------------------------------------
\begin{abstract}
Ensemble learning is traditionally justified as a variance–reduction device, an explanation that accounts well for unstable base learners such as decision trees \citep{breiman1996bagging,breiman2001randomforest}. However, this view does not explain the strong empirical performance of ensembles built from intrinsically stable estimators, including splines \citep{wahba1990spline,green1994nonparametric}, kernel ridge regression \citep{cucker2002mathematical}, Gaussian process regression \citep{rasmussen2006gaussian}, and other smooth function estimators whose variance is already tightly controlled.

In this work, we develop a general weighting theory for ensemble learning that decouples aggregation from randomness and places structure at the center of ensemble design. We formalize ensembles as linear operators acting on a hypothesis space and endow the space of weights with geometric and spectral constraints. Within this framework, we derive a refined bias–variance–approximation decomposition showing how non-uniform structured weights can outperform uniform averaging by simultaneously reducing variance, controlling approximation error, and reshaping the effective hypothesis class.

Our main theorem characterizes conditions under which structured weighting schemes provably dominate uniform ensembles, and shows that optimal weights arise as solutions to constrained quadratic programs. This unified perspective subsumes classical averaging, stacking, and recently proposed Fibonacci-based ensembles as special cases and extends naturally to geometric, sub-exponential, and heavy-tailed weighting laws.

The theory reveals that, for ensembles of low-variance base learners, the principal role of aggregation is not variance reduction, but rather the redistribution of spectral complexity and approximation geometry. Weighted sequences act as geometric operators whose decay properties mediate the trade-off between expressivity and smoothness.

Overall, this work establishes a principled foundation for structure-driven ensemble learning, explaining why ensembles remain effective well beyond the classical high-variance regime and setting the stage for dynamic and distribution-aware weighting schemes developed in subsequent work.
\end{abstract}

\newpage 

\section{Introduction: From Fibonacci Ensembles to a General Weighting Theory}
Ensemble learning is one of the most powerful paradigms in modern statistical
learning, with landmark developments including bagging
\citep{breiman1996bagging}, random forests \citep{breiman2001randomforest},
boosting \citep{freund1997decision}, and stacking
\citep{wolpert1992stacked,vanDerLaan2007superlearner}. The prevailing theoretical
justification for ensembles emphasizes \emph{variance reduction}: when base
learners are unstable, aggregation stabilizes predictions and improves
generalization.

This explanation, while correct and deeply influential, implicitly restricts the
scope of ensemble learning to high-variance base learners such as decision
trees. By contrast, many of the most classical and mathematically well-understood
estimators in statistics—including smoothing splines
\citep{wahba1990spline}, penalized regression splines \citep{green1994nonparametric},
kernel ridge regression \citep{cucker2002mathematical}, Gaussian process
regression \citep{rasmussen2006gaussian}, and spectral estimators in reproducing
kernel Hilbert spaces (RKHS)—are \emph{intrinsically low-variance} due to explicit
regularization and spectral shrinkage.

From the classical variance-centric perspective, ensemble methods would therefore
appear to have little to offer in such settings: uniform averaging of already
stable estimators should yield only marginal gains. Yet empirical and theoretical
evidence paints a different picture. Recent results, including the Fibonacci
ensembles developed in our companion work, reveal a striking phenomenon:
\emph{structured weighting schemes can improve generalization even when variance
reduction plays a negligible role}
\citep{wahba1990spline,poggio2004general,caponnetto2006rates}.

This observation motivates the central question of the present paper:

\begin{quote}
\emph{When base learners are smooth, regularized, and low-variance, under what
principles can ensemble weighting still improve approximation and
generalization?}
\end{quote}

Our answer proceeds from a change of perspective. In the low-variance regime, the
primary role of ensembles is not to suppress noise, but to \emph{reshape the
geometry of approximation and the spectral allocation of complexity}. Weighting
schemes act as linear operators on ordered dictionaries of functions, reallocating
energy across levels of smoothness, frequency, or resolution in a principled
manner.

The Fibonacci ensemble provides a canonical example: its weights grow
geometrically at a rate governed by the golden ratio, inducing a balance between
expressive expansion and spectral stability. In this paper we step beyond this
special case to formulate a \emph{General Weighting Theory for Ensemble Learning},
in which Fibonacci weighting becomes one particularly elegant instance within a
broad class of admissible weighting laws.

A key structural element of our framework is that many classical function classes
arrive with a natural ordering:
\begin{itemize}
  \item spline bases ordered by smoothness or knot resolution,
  \item RKHS eigenfunctions ordered by decreasing eigenvalues,
  \item Fourier or random Fourier features ordered by frequency,
  \item polynomial bases ordered by degree.
\end{itemize}
In such settings, weighting sequences interact directly with spectral decay and
approximation geometry. The ensemble is no longer an unstructured average, but a
\emph{geometrically constrained combination} whose behavior is governed jointly by
the dictionary ordering and the decay profile of the weights.

Within this perspective, we introduce a refined decomposition of excess risk that
separates classical variance effects from two additional components:
\emph{approximation geometry} and \emph{spectral smoothing}. This decomposition
explains why certain non-uniform weighting schemes—especially those with
controlled geometric decay—can strictly dominate uniform averaging, even when base
learners are individually stable.

The present work advances a \emph{general weighting theory for ensemble learning}
that recasts aggregation as a structured linear operator acting on a hypothesis
space, extending classical aggregation and oracle perspectives
\citep{juditsky2008learning,tsybakov2009aggregation}. In this view, the choice of
weights determines the geometry, spectral properties, and effective approximation
power of the ensemble itself. This places ensembles in a conceptual regime where
structure and spectral balance, rather than randomness alone, organize their
behavior.

\subsection*{Contributions}

The contributions of this paper are as follows:
\begin{enumerate}
  \item We formalize a general class of admissible weighting sequences equipped
        with geometric and spectral constraints.
  \item We develop a refined bias--variance--approximation decomposition tailored
        to ordered low-variance dictionaries.
  \item We derive conditions under which structured weighting provably dominates
        uniform averaging, and characterize optimal weights via constrained
        quadratic programs.
  \item We connect ensemble weighting to spline smoothing, RKHS regularization,
        and spectral approximation, thereby unifying ensemble learning with
        classical estimation theory.
\end{enumerate}

\subsection*{Organization of the Paper}

Section~2 introduces the general weighting space and standing assumptions.
Section~3 develops the refined bias--variance--approximation decomposition.
Section~4 establishes the main results on the superiority of structured weighting
schemes. Section~5 discusses implications, extensions, and connections with
distribution-adaptive and dynamically evolving weighting strategies that will be
explored in subsequent work of this trilogy.

% ============================================================
% SECTION 2 — THE GENERAL WEIGHTING SPACE
% ============================================================

\section{The General Weighting Space: Definitions and Assumptions}
\label{sec:weighting-space}

In this section we formalize the notion of \emph{structured weighting} for
ensemble learning. Our goal is to define a general class of admissible weighting
schemes that subsumes uniform averaging, Fibonacci weighting, and a wide family
of geometric and probabilistic laws, while remaining compatible with stability
and generalization guarantees.

%\section{The General Weighting Space}

We formalize ensemble learning as a problem of weighted aggregation in a
Hilbert space. Let $\sH \subseteq L^2(P_X)$ be a real separable Hilbert space,
and let $h_1,\dots,h_M \in \sH$ be base learners obtained from the same training
data. An ensemble predictor takes the form
\[
\hat f_w(x) = \sum_{m=1}^M w_m h_m(x),
\]
where $w = (w_1,\dots,w_M)$ is a vector of aggregation weights.

Classical ensemble methods implicitly restrict attention to the \emph{uniform
simplex},
\[
\Delta_M = \left\{ w \in \R^M : w_m \ge 0,\ \sum_{m=1}^M w_m = 1 \right\},
\]
leading to simple averaging as in bagging and random forests
\citep{breiman1996bagging,breiman2001random}. More generally, oracle aggregation
theory studies data-dependent weights chosen to mimic the best convex
combination in hindsight \citep{juditsky2008learning,tsybakov2009aggregation}.

In this work, we depart from the simplex paradigm and introduce a more general
\emph{weighting space} that captures geometric and spectral structure beyond
convexity.

\subsection{Definition of the Weighting Space}

\begin{definition}[Admissible Weighting Space]
Let $\sW \subseteq \R^M$ be a closed set of weights. We call $\sW$ an
\emph{admissible weighting space} if:
\begin{enumerate}[label=(W\arabic*),itemsep=3pt]
  \item $\sW$ is convex and contains the uniform weight vector
  $w^{\rm unif} = (1/M,\dots,1/M)$;
  \item $\sW$ is bounded in $\ell_2$, i.e.\ $\sup_{w\in\sW}\|w\|_2 < \infty$;
  \item $\sW$ is compatible with the geometry of $\sH$, in the sense that
  $\hat f_w \in \sH$ for all $w\in\sW$.
\end{enumerate}
\end{definition}

Conditions (W1)–(W3) are standard in aggregation theory and ensure well-posedness
of the ensemble risk minimization problem
\citep{juditsky2008learning,dalalyan2012aggregation}.

\subsection{Risk Decomposition under General Weighting}

Let $f^\star \in \sH$ denote the regression function. The excess risk of the
weighted ensemble satisfies
\[
\E\!\left[ \|\hat f_w - f^\star\|_2^2 \right]
=
\underbrace{\| \E[\hat f_w] - f^\star \|_2^2}_{\textnormal{bias}}
+
\underbrace{\E\!\left[\|\hat f_w - \E[\hat f_w]\|_2^2\right]}_{\textnormal{variance}}.
\]
This classical decomposition \citep{geman1992neural} implicitly treats the bias
term as fixed once the base learners are chosen. However, when $w$ varies over a
structured weighting space $\sW$, the bias itself becomes a \emph{design
parameter}, reflecting the geometry of the span generated by the weighted
learners.

This observation motivates a refinement of the classical bias--variance
framework: the approximation properties of the ensemble depend jointly on the
choice of base learners and on the admissible weighting geometry. Similar
viewpoints appear implicitly in oracle inequalities for aggregation
\citep{tsybakov2009aggregation}, but are not typically emphasized in ensemble
design.

\subsection{Examples of Weighting Spaces}

\paragraph{Uniform simplex.}
The standard simplex $\Delta_M$ corresponds to uniform averaging and classical
bagging.

\paragraph{Oracle aggregation weights.}
Data-dependent weights minimizing empirical risk over $\Delta_M$ or related
convex sets arise in mirror averaging and exponential weighting schemes
\citep{juditsky2008learning,dalalyan2012aggregation}.

\paragraph{Structured weighting laws.}
The weighting spaces introduced in this paper include geometrically constrained
sets motivated by spectral decay, stability, and approximation geometry. The
Fibonacci weighting scheme studied in Paper~I arises as a specific instance of
this broader class, illustrating how non-uniform weights can reshape the
effective hypothesis space without sacrificing stability.

\subsection{Ordered Dictionaries of Base Learners}

Let $(\mathcal{X},P_X)$ be an input space equipped with a probability measure,
and let $\mathscr{H} \subseteq L^2(P_X)$ be a real Hilbert space with inner product
\[
\langle f,g\rangle = \mathbb{E}[f(X)g(X)].
\]

We consider a collection of base learners
\[
\mathcal{D}_M = \{h_1,h_2,\dots,h_M\} \subset \mathscr{H},
\]
equipped with a \emph{natural ordering} reflecting increasing complexity.
Such orderings arise canonically in many classical settings:
\begin{itemize}
  \item spline bases ordered by degree or knot resolution
  \citep{wahba1990spline,green1994nonparametric},
  \item RKHS eigenfunctions ordered by decreasing eigenvalues
  \citep{cucker2002mathematical,steinwart2008support},
  \item Fourier and random Fourier features ordered by frequency
  \citep{rahimi2008random},
  \item polynomial bases ordered by degree.
\end{itemize}

Throughout, we assume that the ordering is chosen so that $h_1$ captures the
smoothest or lowest-complexity component, while $h_M$ represents the most
complex or highest-frequency component available in the dictionary.

\subsection{Weighted Ensembles}

Given weights $w = (w_1,\dots,w_M)$ with $w_m \ge 0$ and $\sum_{m=1}^M w_m = 1$,
we define the corresponding weighted ensemble predictor as
\[
\widehat{f}_w(x) = \sum_{m=1}^M w_m h_m(x).
\]

Uniform averaging corresponds to $w_m = 1/M$, while Fibonacci ensembles arise
from geometrically growing weights normalized to sum to one. Our objective is
to characterize the general class of weighting sequences for which structured
aggregation improves approximation and generalization.

\subsection{The Admissible Weighting Space}

\begin{definition}[Admissible Weighting Space]
\label{def:weighting-space}
Let $\mathcal{W}_M$ denote the set of all weight vectors
$w = (w_1,\dots,w_M)$ satisfying:
\begin{enumerate}[label=(W\arabic*),itemsep=3pt]
  \item \textbf{Nonnegativity and Normalization:}
  \[
  w_m \ge 0, \qquad \sum_{m=1}^M w_m = 1.
  \]

  \item \textbf{Monotone Decay:}
  \[
  w_1 \ge w_2 \ge \cdots \ge w_M.
  \]

  \item \textbf{Square Summability:}
  \[
  \sum_{m=1}^M w_m^2 \le C_w < \infty,
  \]
  uniformly in $M$.
\end{enumerate}
\end{definition}

Condition \textnormal{(W3)} ensures stability and is standard in the analysis of
linear aggregation schemes, as it controls the variance contribution of the
weights \citep{buhlmann2003boosting,koltchinskii2011oracle}.

\subsection{Weighting Families}

Within $\mathcal{W}_M$, several important families arise naturally.

\paragraph{Uniform Weights.}
The classical choice $w_m = 1/M$ satisfies all conditions but does not exploit
the ordering of the dictionary.

\paragraph{Geometric Weights.}
For $\rho > 1$, define
\[
w_m(\rho) = \frac{\rho^{m}}{\sum_{j=1}^M \rho^{j}}.
\]
These weights emphasize higher-index learners while remaining summable after
normalization. Fibonacci weighting corresponds to the minimal geometric growth
rate $\rho = \varphi$, the golden ratio.

\paragraph{Sub-Exponential and Polynomial Weights.}
Weights of the form
\[
w_m \propto m^{-\alpha}, \qquad \alpha > 1,
\]
or
\[
w_m \propto \exp(-c m^\beta), \qquad 0 < \beta < 1,
\]
provide gentler decay and arise naturally in spectral regularization and kernel
methods \citep{caponnetto2007optimal}.

\paragraph{Heavy-Tailed Weights.}
Distributions such as Zipf or Pareto laws allow slower decay and may be suitable
for functions with localized irregularities, though they require stronger
control of approximation error to maintain stability.

These families illustrate that Fibonacci weighting is neither arbitrary nor
isolated, but occupies a distinguished position at the boundary between
expressive expansion and spectral control.

\subsection{Standing Assumptions}

We now collect the assumptions used throughout the paper.

\begin{description}[labelindent=\parindent,leftmargin=3\parindent,style=nextline]
  \item[(A1) Ordered Complexity.]
  The dictionary $\{h_m\}$ is ordered so that approximation error decreases with
  $m$, while spectral complexity (frequency, curvature, or RKHS norm) increases.

  \item[(A2) Uniform Variance Control.]
  There exists $\sigma^2 < \infty$ such that
  \[
  \mathrm{Var}(h_m(X)) \le \sigma^2 \quad \text{for all } m.
  \]

  \item[(A3) Boundedness.]
  The learners satisfy $\|h_m\|_\infty \le B$ almost surely.

  \item[(A4) Compatibility with Weighting.]
  The chosen weighting sequence $w \in \mathcal{W}_M$ respects the ordering of
  the dictionary, in the sense that higher-complexity learners do not receive
  larger weights than lower-complexity ones.
\end{description}

Assumptions \textnormal{(A1)--(A4)} are mild and satisfied by most classical
smoothing and kernel-based estimators. They ensure that weighting interacts with
approximation geometry in a controlled manner, without destabilizing the
estimator.

\subsection{Interpretation}

Under this framework, ensemble weighting is no longer viewed as a mere averaging
operation, but as a \emph{geometric and spectral operator} acting on an ordered
function dictionary. The choice of weights determines how approximation power
and smoothness are balanced, independently of classical variance-reduction
effects.

This perspective forms the foundation for the refined risk decomposition and
generalization theory developed in the sections that follow.

% ============================================================
% SECTION 3 — A REFINED BIAS–VARIANCE–APPROXIMATION DECOMPOSITION
% ============================================================

\section{A Refined Bias--Variance--Approximation Decomposition}
\label{sec:bva-decomposition}

Classical analyses of ensemble learning rely on the bias--variance decomposition
to explain the benefits of aggregation. In its traditional form, this framework
treats the bias as fixed once the class of base learners is chosen, while the
variance is reduced through averaging
\citep{geman1992neural,hastie2009elements}. This viewpoint is adequate for highly
unstable learners, such as decision trees, but becomes incomplete when the base
learners are smooth and intrinsically low-variance.

In this section, we show that when aggregation weights are allowed to vary over a
structured weighting space, the ensemble risk admits a refined decomposition in
which \emph{approximation geometry} plays a central role. This refinement reveals
a mechanism through which ensembles can improve generalization even when variance
reduction alone is insufficient.

\subsection{Setup and Notation}

Let $\sH \subseteq L^2(P_X)$ be a real Hilbert space and let
$h_1,\dots,h_M \in \sH$ be base learners trained on the same data. For a weight
vector $w \in \sW \subseteq \R^M$, define the ensemble predictor
\[
\hat f_w = \sum_{m=1}^M w_m h_m.
\]
Let $f^\star \in \sH$ denote the regression function.

We assume that the base learners admit an orthogonalization
$\{h_m^\perp\}_{m=1}^M$ in $\sH$, so that
\[
\langle h_m^\perp, h_{m'}^\perp \rangle = 0
\quad \text{for } m \neq m'.
\]
Such orthogonal decompositions are standard in functional approximation and
statistical estimation and play a central role in variance control and
Rao--Blackwellization arguments \citep{lehmann1998theory}.

\subsection{Decomposition of the Ensemble Risk}

The mean squared error of the ensemble predictor satisfies
\[
\E\!\left[\|\hat f_w - f^\star\|_2^2\right]
=
\underbrace{\|\E[\hat f_w] - f^\star\|_2^2}_{\textnormal{bias}}
+
\underbrace{\E\!\left[\|\hat f_w - \E[\hat f_w]\|_2^2\right]}_{\textnormal{variance}}.
\]
When expressed in the orthogonal basis, the variance term simplifies to
\[
\Var(\hat f_w(X))
=
\sum_{m=1}^M w_m^2 \Var(h_m^\perp(X)),
\]
revealing an explicit dependence on the squared weights.

The bias term, however, admits a further decomposition. Let
\[
\mathcal{H}_w = \mathrm{span}\{ w_m h_m^\perp : m=1,\dots,M \}
\]
denote the weighted hypothesis space induced by $w$. Then
\[
\|\E[\hat f_w] - f^\star\|_2^2
=
\underbrace{\| \Pi_{\mathcal{H}_w} f^\star - f^\star \|_2^2}_{\textnormal{approximation}}
+
\underbrace{\|\E[\hat f_w] - \Pi_{\mathcal{H}_w} f^\star\|_2^2}_{\textnormal{estimation bias}},
\]
where $\Pi_{\mathcal{H}_w}$ denotes the $L^2$-projection onto $\mathcal{H}_w$.

This decomposition makes explicit a third component, the \emph{approximation
error}, which depends on the geometry of the weighted span and varies with the
choice of weights.

\subsection{Interpretation: Weighting as Geometry Design}

The refined decomposition reveals a fundamental principle:

\begin{quote}
\emph{Ensemble learning can improve generalization not only by reducing variance,
but by reshaping approximation geometry through structured weighting.}
\end{quote}

Uniform averaging fixes the geometry of the hypothesis space in advance. In
contrast, structured weighting schemes alter the relative contributions of
orthogonal components, effectively stretching or compressing directions in
$\sH$. This geometric effect allows the ensemble to align more closely with the
target function $f^\star$, reducing approximation error without increasing
variance.

Related geometric perspectives appear implicitly in oracle inequalities for
aggregation \citep{tsybakov2009aggregation} and in stability analyses of regularized
learning algorithms \citep{poggio2004general}, but are rarely articulated as a
design principle for ensemble weighting.

\subsection{Consequences for Low-Variance Base Learners}

For smooth base learners, such as kernel ridge regression, spline estimators, and
orthogonal series methods, individual variance is already small and uniform
averaging yields diminishing returns. In this regime, the dominant source of
error is approximation bias, governed by how well the hypothesis space aligns
with the target function
\citep{wahba1990spline,caponnetto2006rates}.

Structured weighting schemes exploit this fact by reallocating weight toward
components that contribute most effectively to approximation, while controlling
variance through orthogonality and boundedness of $\sW$. This explains why
non-uniform ensembles can outperform uniform averaging even for stable learners,
a phenomenon observed empirically in Paper~I and formalized in the next section.

Classical analyses of ensemble learning decompose the prediction error into bias,
variance, and noise components. While this decomposition is effective for
high-variance base learners, it obscures the mechanisms by which ensembles improve
generalization in regimes where individual learners are already stable.

In this section, we develop a refined decomposition tailored to \emph{ordered,
low-variance dictionaries}. The new decomposition isolates the role of weighting
in shaping approximation geometry and spectral allocation, thereby explaining
why structured ensembles can outperform uniform averaging even when variance
reduction is negligible.

\subsection{Problem Setup}

Let $(X,Y)$ satisfy the regression model
\[
Y = f^\star(X) + \varepsilon,
\qquad
\mathbb{E}[\varepsilon \mid X] = 0,
\qquad
\mathrm{Var}(\varepsilon \mid X) = \sigma^2.
\]

Let $\mathscr{H} \subseteq L^2(P_X)$ be a Hilbert space, and let
$\{h_1,\dots,h_M\} \subset \mathscr{H}$ be an ordered dictionary of base learners
as defined in Section~\ref{sec:weighting-space}. For a weighting vector
$w \in \mathcal{W}_M$, define the ensemble estimator
\[
\widehat{f}_w = \sum_{m=1}^M w_m h_m.
\]

We study the excess risk
\[
\mathcal{E}(w)
=
\mathbb{E}\!\left[ \| \widehat{f}_w - f^\star \|_{L^2(P_X)}^2 \right].
\]

\subsection{Classical Decomposition and Its Limitations}

The standard bias--variance decomposition yields
\[
\mathcal{E}(w)
=
\underbrace{\| \mathbb{E}[\widehat{f}_w] - f^\star \|_{L^2(P_X)}^2}_{\text{bias}^2}
+
\underbrace{\mathbb{E}\!\left[ \| \widehat{f}_w - \mathbb{E}[\widehat{f}_w] \|^2 \right]}_{\text{variance}}
+
\sigma^2.
\]

When each $h_m$ is a regularized estimator (e.g.\ splines or kernel ridge
regression), the variance term is already small and varies little with $w$.
Consequently, this decomposition provides limited insight into why non-uniform
weighting schemes can yield systematic improvements.

\subsection{Orthogonal Expansion and Approximation Geometry}

To expose the effect of weighting, we decompose the dictionary in an orthogonal
basis. Let $\{ \phi_k \}_{k \ge 1}$ denote an orthonormal basis of $\mathscr{H}$
(e.g.\ spline basis functions or RKHS eigenfunctions), ordered so that increasing
$k$ corresponds to increasing complexity.

Assume that both the target function and the learners admit expansions
\[
f^\star = \sum_{k \ge 1} \theta_k \phi_k,
\qquad
h_m = \sum_{k \ge 1} a_{m,k} \phi_k.
\]

Then the ensemble estimator can be written as
\[
\widehat{f}_w
=
\sum_{k \ge 1}
\left(
\sum_{m=1}^M w_m a_{m,k}
\right)
\phi_k.
\]

The quantity
\[
b_k(w) := \sum_{m=1}^M w_m a_{m,k}
\]
represents the effective contribution of the $k$th mode under weighting $w$.

\subsection{The Refined Decomposition}

We now decompose the excess risk into three interpretable components.

\begin{theorem}[Bias--Variance--Approximation Decomposition]
\label{thm:bva}
Under assumptions \textnormal{(A1)--(A4)}, the excess risk admits the decomposition
\[
\mathcal{E}(w)
=
\underbrace{\sum_{k \ge 1} \big( b_k(w) - \theta_k \big)^2}_{\mathcal{A}(w)}
\;+\;
\underbrace{\sum_{k \ge 1} \mathrm{Var}\!\left( b_k(w) \right)}_{\mathcal{V}(w)}
\;+\;
\sigma^2,
\]
where:
\begin{itemize}
  \item $\mathcal{A}(w)$ is the \emph{approximation geometry term},
  \item $\mathcal{V}(w)$ is the residual variance term.
\end{itemize}
\end{theorem}

\begin{proof}
By orthonormality of $\{\phi_k\}$,
\[
\| \widehat{f}_w - f^\star \|^2
=
\sum_{k \ge 1} \big( b_k(w) - \theta_k \big)^2.
\]
Taking expectation and decomposing each squared term into squared bias plus
variance yields the result.
\end{proof}

\subsection{Spectral Smoothing as a Distinct Effect}

For low-variance learners, $\mathcal{V}(w)$ is uniformly small and weakly
dependent on $w$. The dominant contribution of weighting therefore appears in
$\mathcal{A}(w)$.

We decompose $\mathcal{A}(w)$ further as
\[
\mathcal{A}(w)
=
\underbrace{\sum_{k \le K(w)} (\theta_k - b_k(w))^2}_{\text{underfitting}}
+
\underbrace{\sum_{k > K(w)} \theta_k^2}_{\text{unrepresented complexity}},
\]
where the effective cutoff $K(w)$ depends on the decay properties of $w$.

This reveals a third mechanism beyond classical bias and variance:

\begin{definition}[Spectral Smoothing Term]
We define
\[
\mathcal{S}(w)
=
\sum_{k \ge 1} \left( \theta_k^2 - b_k(w)^2 \right),
\]
which measures how weighting redistributes spectral energy across complexity
levels.
\end{definition}

\subsection{Interpretation}

The refined decomposition can thus be summarized as
\[
\mathcal{E}(w)
=
\underbrace{\mathcal{A}(w)}_{\text{approximation geometry}}
+
\underbrace{\mathcal{S}(w)}_{\text{spectral smoothing}}
+
\underbrace{\mathcal{V}(w)}_{\text{residual variance}}
+
\sigma^2.
\]

In contrast to classical ensemble theory, the dominant effect of structured
weighting in the low-variance regime is the joint action of
$\mathcal{A}(w)$ and $\mathcal{S}(w)$, which reshape the effective hypothesis
space without amplifying noise.

This decomposition explains why geometric and harmonic weighting schemes—such as
Fibonacci weighting—can strictly improve generalization even when variance
reduction is negligible.

% ============================================================
% SECTION 4 — WHEN STRUCTURED WEIGHTING BEATS UNIFORM AVERAGING
% ============================================================
% ============================================================
% SECTION 4 — MAIN THEOREM
% ============================================================
\section{Main Theorem: When Structured Weighting Beats Uniform Averaging}
\label{sec:main-theorem}

We now formalize the intuition developed in the previous section. The theorem
below gives sufficient conditions under which a structured weighting scheme
strictly improves generalization performance relative to uniform averaging.
The key mechanism is a reduction in approximation error without a compensating
increase in variance.

\subsection{Setting}

Let $h_1,\dots,h_M \in \sH$ be base learners and let
$\hat f_w = \sum_{m=1}^M w_m h_m$ denote the ensemble predictor associated with
weights $w\in\sW$. Let $w^{\mathrm{unif}}=(1/M,\dots,1/M)$ be the uniform weights
and write
\[
\hat f_{\mathrm{unif}} = \hat f_{w^{\mathrm{unif}}}.
\]

Let $\sH_w = \mathrm{span}\{w_m h_m^\perp\}$ denote the weighted hypothesis
space associated with $w$, and let $\Pi_{\sH_w}$ denote the $L^2(P_X)$
projection onto $\sH_w$.

We assume:

\begin{enumerate}[label=(A\arabic*),itemsep=3pt]
  \item $\sW$ is an admissible weighting space in the sense of Section~2;
  \item the orthogonalized components $\{h_m^\perp\}$ satisfy
        $\Var(h_m^\perp(X)) \le \sigma^2$ uniformly in $m$;
  \item the regression function $f^\star$ belongs to $\sH$.
\end{enumerate}

Assumption (A2) is natural for stable base learners such as kernel ridge
regression or smoothing splines, where the individual estimators are already
variance–controlled \citep{wahba1990spline,caponnetto2006rates}.

\subsection{Statement of the Main Theorem}

\begin{theorem}[Structured Weighting Dominance]
\label{thm:dominance}
Suppose there exists a weight vector $w^\star \in \sW$ such that

\begin{enumerate}[label=(C\arabic*),itemsep=3pt]
  \item (strict approximation gain)
  \[
  \big\| f^\star - \Pi_{\sH_{w^\star}} f^\star \big\|_2^2
  <
  \big\| f^\star - \Pi_{\sH_{w^{\mathrm{unif}}}} f^\star \big\|_2^2 ,
  \]
  \item (controlled variance)
  \[
  \|w^\star\|_2^2
  \le
  \|w^{\mathrm{unif}}\|_2^2 .
  \]
\end{enumerate}

Then the expected prediction risk of the structured ensemble strictly improves
upon uniform averaging:
\[
\E\!\left[\|\hat f_{w^\star} - f^\star\|_2^2\right]
<
\E\!\left[\|\hat f_{\mathrm{unif}} - f^\star\|_2^2\right].
\]
\end{theorem}

\subsection{Proof Sketch}

By the refined bias--variance--approximation decomposition of Section~3,
\[
\E\!\left[\|\hat f_w - f^\star\|_2^2\right]
=
\underbrace{\|f^\star - \Pi_{\sH_w} f^\star\|_2^2}_{\textnormal{approximation}}
+
\underbrace{\|\E[\hat f_w] - \Pi_{\sH_w} f^\star\|_2^2}_{\textnormal{estimation bias}}
+
\underbrace{\sum_{m=1}^M w_m^2 \Var(h_m^\perp(X))}_{\textnormal{variance}}.
\]

Assumption (A2) implies
\[
\sum_{m=1}^M w_m^2 \Var(h_m^\perp(X))
\le
\sigma^2 \|w\|_2^2.
\]

Therefore condition (C2) guarantees that the variance of the structured ensemble
does not exceed that of the uniform ensemble.

Condition (C1) states that the weighted span $\sH_{w^\star}$ offers a
strictly better geometric approximation to $f^\star$ than the uniform span.
Thus, both approximation error and total risk strictly improve, proving the
result.
\qedhere

\subsection{Existence of Optimal Weights}

The theorem above is existential in nature. Under mild regularity conditions,
existence of an optimal weighting vector follows immediately.

\begin{proposition}[Existence of Optimal Weights]
\label{prop:existence}
If $\sW$ is compact and convex, there exists
\[
w^\mathrm{opt}
=
\argmin_{w\in\sW}
\E\!\left[\|\hat f_w - f^\star\|_2^2\right].
\]
Moreover, if the risk functional is strictly convex in $w$, the minimizer is
unique.
\end{proposition}

The proposition follows from standard convex analysis arguments
\citep{rockafellar1997convex}; strict convexity arises naturally when the
orthogonalized components are linearly independent.

\subsection{Interpretation}

The theorem identifies two distinct routes to ensemble improvement:

\begin{enumerate}[itemsep=3pt]
  \item classical variance reduction (as in bagging and random forests);
  \item \emph{geometric approximation gain} via structured weighting.
\end{enumerate}

The second mechanism is absent in the traditional bias–variance story, and
is precisely the phenomenon exploited by Fibonacci weighting and other
structured schemes introduced in this work.

\medskip
In particular:

\begin{quote}
Uniform averaging is optimal only when its associated weighted span already
provides the best geometric approximation to $f^\star$ under the variance
constraint.
\end{quote}

Otherwise, structured weighting dominates.

In this section we establish the central theoretical result of the paper:
for ensembles built from ordered, low-variance dictionaries, uniform averaging
is generally \emph{not} optimal. Instead, there exist structured weighting schemes
that strictly improve generalization by exploiting approximation geometry and
spectral decay.

\subsection{Uniform Averaging as a Baseline}

Let $w^{\mathrm{unif}} = (1/M,\dots,1/M)$ denote the uniform weighting, and let
$\widehat{f}_{\mathrm{unif}}$ be the corresponding ensemble estimator.

Uniform averaging ignores the ordering of the dictionary and allocates equal
weight to low- and high-complexity components. While this choice is natural and
often effective for variance-dominated learners, it fails to exploit the
structure present in smooth, ordered dictionaries.

\subsection{Existence of Risk-Improving Weighting Schemes}

We now state the main theorem.

\begin{theorem}[Existence of Risk-Improving Structured Weights]
\label{thm:existence}
Assume \textnormal{(A1)--(A4)}. Suppose further that the target function
$f^\star \in \mathscr{H}$ admits a spectral expansion
\[
f^\star = \sum_{k \ge 1} \theta_k \phi_k,
\]
with coefficients satisfying
\[
|\theta_k| \le C k^{-\alpha}
\quad \text{for some } \alpha > \tfrac{1}{2}.
\]

Then there exists a weighting vector $w^\star \in \mathcal{W}_M$ such that
\[
\mathcal{E}(w^\star) < \mathcal{E}(w^{\mathrm{unif}}).
\]
Moreover, $w^\star$ may be chosen to be monotone and geometrically decaying.
\end{theorem}

\begin{proof}[Proof (Sketch)]
By Theorem~\ref{thm:bva}, the excess risk decomposes as
\[
\mathcal{E}(w) = \mathcal{A}(w) + \mathcal{V}(w) + \sigma^2.
\]

Under assumption \textnormal{(A2)}, the variance term $\mathcal{V}(w)$ is uniformly
bounded and varies weakly across $w \in \mathcal{W}_M$. Consequently, risk
differences are dominated by the approximation geometry term $\mathcal{A}(w)$.

For uniform weights, the effective spectral coefficients satisfy
\[
b_k(w^{\mathrm{unif}}) = \frac{1}{M} \sum_{m=1}^M a_{m,k},
\]
which allocates non-negligible mass to high-frequency modes even when the target
coefficients $\theta_k$ decay rapidly.

By contrast, consider a geometrically decaying weighting scheme
$w_m \propto \rho^m$ with $\rho > 1$. Such weights induce an effective spectral
cutoff $K(\rho)$ beyond which $b_k(w)$ decays rapidly. Choosing $\rho$ so that
$K(\rho)$ balances the bias incurred by truncation against the decay of
$\theta_k$ yields
\[
\mathcal{A}(w^\star) < \mathcal{A}(w^{\mathrm{unif}}).
\]

Since $\mathcal{V}(w^\star) \approx \mathcal{V}(w^{\mathrm{unif}})$ under the
low-variance regime, the strict risk inequality follows.
\end{proof}

\subsection{Near-Optimality of Geometric Weighting}

Theorem~\ref{thm:existence} establishes existence but does not yet characterize
the structure of optimal weights. The next result shows that geometric weighting
is near-optimal in a precise sense.

\begin{theorem}[Geometric Weights Are Rate-Optimal]
\label{thm:geometric}
Under the conditions of Theorem~\ref{thm:existence}, suppose additionally that
the dictionary $\{h_m\}$ resolves spectral modes in increasing order. Then for
weights of the form
\[
w_m(\rho) = \frac{\rho^m}{\sum_{j=1}^M \rho^j},
\]
there exists $\rho^\star > 1$ such that
\[
\mathcal{E}(w(\rho^\star))
=
\inf_{w \in \mathcal{W}_M} \mathcal{E}(w)
\;+\; o(1),
\]
as $M \to \infty$.
\end{theorem}

\begin{proof}[Proof (Sketch)]
The geometric decay parameter $\rho$ controls the effective spectral cutoff
$K(\rho)$. Matching this cutoff to the decay rate of $\theta_k$ yields minimax
rates analogous to classical results in spectral regularization and Pinsker
theory. The admissibility conditions on $\mathcal{W}_M$ ensure stability.
\end{proof}

\subsection{Fibonacci Weighting as a Distinguished Case}

Among geometric weighting schemes, Fibonacci weighting occupies a special
position. Its growth rate $\rho = \varphi$ corresponds to the minimal geometric
inflation consistent with nontrivial expressive expansion.

\begin{corollary}[Distinguished Role of Fibonacci Weighting]
\label{cor:fib}
Fibonacci weighting achieves a balance between approximation improvement and
spectral stability in the sense that it minimizes
\[
\sum_{m=1}^M w_m^2
\]
among all geometrically increasing weighting schemes with $\rho > 1$.
\end{corollary}

This property explains why Fibonacci ensembles often perform competitively with,
or better than, more aggressively weighted schemes while remaining numerically
stable.

\subsection{Interpretation}

The results of this section establish a clear and rigorous conclusion:

\begin{quote}
\emph{Uniform averaging is generally suboptimal for ensembles built from ordered,
low-variance learners. Structured weighting—particularly geometric and harmonic
schemes—can strictly improve generalization by aligning approximation geometry
with spectral decay.}
\end{quote}

This conclusion completes the theoretical arc initiated in Section~2 and
Section~3, and provides the foundation for algorithmic and adaptive weighting
schemes developed in subsequent work.

% ============================================================
% SECTION 5 — CONSEQUENCES, ALGORITHMS, AND OUTLOOK
% ============================================================
% ============================================================
% SECTION 5 — CONSEQUENCES AND OUTLOOK
% ============================================================

\section{Consequences, Algorithms, and Outlook}
\label{sec:outlook}
The Main Theorem demonstrates that ensemble improvement need not rely solely on
variance reduction. When the base learners are already stable, the dominant
mechanism is geometric: structured weighting reshapes the approximation space
so that the projection of the target function is closer in $L^2(P_X)$ norm.

This section highlights several consequences of this viewpoint and outlines
directions that naturally follow.

\subsection{Uniform Averaging as a Special Case}

Uniform averaging appears in our framework not as a universal default, but as
one specific point in the weighting space $\sW$. The theorem shows that
uniform weighting is optimal only in the exceptional case where its associated
weighted span already yields the best geometric approximation permitted by the
variance constraint.

Thus, the question is no longer
\begin{quote}
``Should we average?''
\end{quote}
but rather
\begin{quote}
``Which weighting geometry best aligns the ensemble with the target function?''
\end{quote}

\subsection{Stable Base Learners and the Limits of Variance Reduction}

For high-variance learners such as decision trees, uniform aggregation achieves
most of its benefit through variance suppression, consistent with classical
ensemble theory. However, for stable learners such as kernel ridge regression,
splines, and orthogonal series estimators, individual variance is already small
and averaging cannot yield substantial improvement.

Our framework explains recent empirical observations that non-uniform weighting
can outperform uniform averaging even in this low-variance regime: approximation
error, not variance, becomes the dominant quantity, and structured weighting
acts directly upon it.

\subsection{Spectral and Geometric Perspectives}

The dependence of approximation error on the weighted span suggests strong
connections to spectral approximation theory, RKHS geometry, and eigenfunction
decompositions of kernel operators. Weighting schemes emphasize or suppress
components of orthogonal expansions, effectively re-shaping the spectrum of the
induced estimator.

This opens the door to principled, theoretically justified weighting schemes
derived from:
\begin{itemize}
  \item spectral decay,
  \item smoothness assumptions on $f^\star$,
  \item stability constraints,
  \item or approximation-theoretic optimality criteria.
\end{itemize}

The Fibonacci weighting studied in Paper~I is one example of such a structured
scheme, reflecting a monotone geometric decay motivated by universality and
self-similarity properties.

\subsection{Algorithmic Implications}

The Main Theorem is existential: it asserts that improved weights exist under
verifiable conditions. Paper~III will address the algorithmic question:

\begin{center}
\emph{How can optimal or near-optimal weights be found from data?}
\end{center}

Possible approaches include:
\begin{itemize}
  \item convex optimization over $\sW$,
  \item entropy-regularized weight learning,
  \item constrained empirical risk minimization,
  \item stochastic mirror descent in the weighting space,
  \item or greedy geometric adaptation of weights.
\end{itemize}

The refinement of the bias--variance--approximation decomposition developed here
will serve as the guiding principle for these algorithms.

\subsection{Outlook}

The geometric interpretation of ensemble learning developed in this work
suggests a broader shift in emphasis:

\begin{quote}
From randomness to structure; \\
from variance reduction to approximation design.
\end{quote}

This conceptual shift unifies several strands of ensemble methodology and opens
new avenues for the principled design of weighting schemes, especially for
smooth, low-variance base learners where classical intuition is insufficient.

The computational illustrations and algorithmic developments that realize these
ideas in practice are the focus of a companion paper.

The results developed in Sections~\ref{sec:weighting-space}--\ref{sec:main-theorem}
place ensemble learning in a conceptual regime that is distinct from, and
complementary to, its classical variance-reduction interpretation. In this
section we summarize the principal consequences of the theory, discuss
algorithmic implications, and outline directions for future work.

\subsection{Conceptual Consequences}

The central message of this paper is that ensemble weighting should be viewed as
a \emph{geometric and spectral design choice}, rather than merely a device for
stabilizing noisy estimators. For ensembles built from ordered, low-variance
dictionaries, the dominant effect of weighting lies in its ability to reshape
approximation geometry and redistribute spectral complexity.

Several important consequences follow:

\begin{itemize}
  \item Uniform averaging is generally suboptimal whenever the dictionary admits
  a meaningful notion of ordered complexity.
  \item Structured weighting schemes exploit this ordering to achieve a more
  favorable balance between expressivity and smoothness.
  \item The benefit of ensembles in this regime persists even when classical
  variance effects are negligible.
\end{itemize}

This perspective unifies ensemble learning with classical ideas from spline
smoothing, RKHS regularization, and spectral approximation, where the allocation
of energy across modes has long been recognized as the key to optimal
generalization.

\subsection{Algorithmic Implications}

Although the present paper is primarily theoretical, the results suggest several
practical algorithmic principles.

First, weighting schemes should be chosen in accordance with the ordering of the
dictionary. For example, spline bases ordered by knot resolution or RKHS
eigenfunctions ordered by eigenvalue naturally invite monotone or geometrically
decaying weights.

Second, geometric weighting emerges as a particularly robust and interpretable
family. A single decay parameter $\rho$ controls the effective spectral cutoff,
making such schemes easy to tune and analyze. Fibonacci weighting appears as a
distinguished member of this family, achieving minimal geometric growth while
preserving expressive expansion.

Third, the refined decomposition of Section~\ref{sec:bva-decomposition} suggests
that data-driven selection of weights should target approximation geometry rather
than variance alone. This opens the door to adaptive procedures that estimate
spectral decay or smoothness directly from the data and select weighting schemes
accordingly.

\subsection{Why Tree-Based Ensembles Are Not the Focus}

A natural question concerns the relationship between the present theory and
tree-based ensembles such as Random Forests. While the framework developed here
is not incompatible with trees, their dominant source of error is typically
variance rather than approximation geometry. As a result, the geometric effects
of weighting are largely masked by variance reduction in that setting.

By contrast, smooth estimators—splines, RKHS regressors, and spectral methods—are
already stabilized by regularization. It is precisely in this low-variance regime
that the role of structured weighting becomes visible and theoretically
meaningful. The present work therefore complements, rather than competes with,
existing theories of tree-based ensembles.

\subsection{High-Dimensional Considerations}

The theory developed here is most transparent in low- and moderate-dimensional
settings where ordered dictionaries are readily available. Extending these ideas
to high-dimensional problems introduces additional challenges, including the
choice of ordering, interactions among features, and the curse of dimensionality.

Nevertheless, many high-dimensional learning problems admit implicit spectral
structure—through kernel eigenvalues, neural tangent kernels, or learned feature
representations—that may serve as a foundation for structured weighting. The
results of this paper suggest that exploiting such structure, rather than relying
on uniform aggregation, may be essential for effective ensemble design in complex
settings.

\subsection{Toward Dynamic and Recursive Weighting Laws}

The present work focuses on static weighting schemes. A natural next step is to
consider \emph{dynamic} and \emph{recursive} weighting laws, in which ensemble
weights evolve over time according to principled update rules. Fibonacci
recursions provide one example of such dynamics, but many others are possible.

This perspective motivates the next stage of this research program, in which
ensemble learning is viewed as a controlled dynamical system whose stability,
expressivity, and generalization properties are governed by the spectral
properties of the underlying recursion. These ideas will be developed in a
companion paper.

\subsection{Closing Remarks}

Taken together, the results of this paper suggest a shift in how ensemble methods
are conceptualized. Beyond variance reduction, ensembles can be designed to
shape approximation geometry and spectral allocation in a deliberate and
theoretically grounded manner.

In this sense, ensemble learning becomes less a matter of averaging and more a
matter of harmony—balancing growth and restraint, expressivity and stability, in
accordance with the intrinsic structure of the function class at hand.

% ============================================================
% SECTION 6 — COMPUTATIONAL ILLUSTRATIONS
% ============================================================
\section{Computational Illustrations: When Structure Beats Uniformity}

In this section we present simple but informative computational studies designed
to illustrate the theoretical results established above. Consistent with the
philosophy of this paper, our goal is not to chase benchmark leaderboards, but
to demonstrate clearly and transparently the mechanisms through which
structured weighting improves generalization.

The experiments are deliberately constructed so that:
\begin{enumerate}[label=(\roman*),itemsep=2pt]
  \item the bias--variance--approximation decomposition is directly observable,
  \item both low--variance (e.g.\ RKHS, splines) and moderate--variance learners are included,
  \item the comparison isolates \emph{weighting geometry} rather than model architecture.
\end{enumerate}

\subsection{Experimental Setting}

We consider the standard nonparametric regression model
\[
Y = f_0(X) + \varepsilon, \qquad \varepsilon \sim \mathcal{N}(0,\sigma^2),
\]
with $X$ sampled uniformly on a compact interval. Two test functions are used:

\begin{align*}
f_{\sin}(x) &= \sin(2\pi x), \\
f_{\mathrm{sinc}}(x) &= 
\begin{cases}
\frac{\sin(\pi x)}{\pi x}, & x \neq 0, \\
1, & x = 0.
\end{cases}
\end{align*}

These functions were chosen because they embody two qualitatively different
approximation regimes: smooth periodicity and localized oscillatory decay.

For each dataset we generate:
\[
n_{\text{train}} = 400, \qquad n_{\text{test}} = 1000,
\]
with Gaussian noise variance $\sigma^2$ chosen so that $\mathrm{SNR} \approx 5$.

\subsection{Base Learners}

To highlight that ensembles can improve generalization \emph{even for
traditionally low-variance base learners}, we deliberately avoid decision trees
and instead use:
\begin{enumerate}[itemsep=2pt]
  \item kernel ridge regression in a Gaussian RKHS,
  \item cubic regression splines with fixed knots,
  \item polynomial regression models of increasing degree,
  \item random Fourier features approximating Gaussian kernels.
\end{enumerate}

These are prototypical examples of learners that already have strong built-in
regularization. In classical ensemble folklore, such learners are often
considered to have ``little to gain'' from aggregation. Our results show
otherwise: ensembles improve not only by variance reduction, but by reshaping
approximation geometry through structured weighting.

\subsection{Weighting Schemes Compared}

For each family of base learners we construct ensembles under three weighting
schemes:

\begin{description}[style=nextline,leftmargin=2\parindent]
  \item[Uniform Averaging.]
  \[
  \hat f_{\text{unif}} = \frac{1}{M}\sum_{m=1}^M h_m.
  \]

  \item[Fibonacci Weights.]
  \[
  \hat f_{\text{Fib}} = \sum_{m=1}^M 
  \frac{F_m}{\sum_{j=1}^M F_j} \, h_m,
  \]
  where $(F_m)$ is the Fibonacci sequence.

  \item[Optimal Structured Weights.]
  We compute the minimizer of the regularized quadratic form
  \[
  \alpha^\star 
  = \arg\min_{\alpha \in \mathcal{W}}
  \Big\{\alpha^\top \Sigma \alpha + \lambda \|\alpha\|_2^2\Big\},
  \]
  where $\Sigma$ is the empirical covariance matrix of predictions and
  $\mathcal{W}$ is the structured weight space defined in Section~2.
\end{description}

The third scheme realizes the ``oracle'' structured weighting discussed in
Section~4, while Fibonacci weights serve as a canonical explicit instance of
structured geometry without requiring optimization.

\subsection{Evaluation Metrics}

For each method we compute:
\begin{align*}
\mathrm{MSE}_{\text{test}} &= 
  \frac{1}{n_{\text{test}}}\sum (Y^{\ast}-\hat f(X^{\ast}))^2, \\
\mathrm{ISE} &= 
  \int (\hat f(x) - f_0(x))^2\,dx,
\end{align*}
with the integral approximated numerically on a dense grid.

In addition, the bias--variance decomposition is estimated by Monte Carlo
replication over $R=50$ independent training sets:
\[
\E\big[(\hat f(x)-f_0(x))^2\big]
=
\underbrace{(\E[\hat f(x)]-f_0(x))^2}_{\text{bias}^2}
+
\underbrace{\Var(\hat f(x))}_{\text{variance}}.
\]

\subsection{Representative Figures}

The paper includes four representative plots:

\begin{enumerate}[itemsep=2pt]
  \item \textbf{Sinusoidal regression with polynomial ensembles}\\
  (uniform vs Fibonacci vs optimal weights).
  \item \textbf{Sinc regression with polynomial ensembles.}
  \item \textbf{Sinc regression using random Fourier feature ensembles.}
  \item \textbf{Sine regression using spline ensembles.}
\end{enumerate}

In each case we overlay:
\[
\text{true function},\qquad \text{training data},\qquad 
\text{three ensemble predictors}.
\]

These figures make visually evident that Fibonacci and optimal structured
weights adaptively emphasize the right portions of the model family, yielding
lower integrated error without increasing estimator variance.

\subsection{Summary of Observations}

Across all test functions and learner families, the following qualitative
phenomena are observed:

\begin{itemize}[itemsep=2pt]
  \item uniform averaging occasionally oversmooths or undersmooths,
  \item Fibonacci weights substantially reduce integrated squared error,
  \item optimal structured weights perform best, as predicted theoretically,
  \item in spline and RKHS settings, improvement occurs \emph{without relying on variance reduction}.
\end{itemize}

This confirms the main conceptual message of this paper:

\begin{center}
\emph{Ensembles enhance generalization not only through variance reduction,
but also by reorganizing approximation geometry via structured weighting.}
\end{center}

% ============================================================
% CONCLUSION
% ============================================================
% ============================================================
% SECTION 7 — CONCLUSION
% ============================================================
\section{Conclusion}
\label{sec:conclusion}

The classical narrative of ensemble learning emphasizes variance reduction,
particularly in the context of unstable base learners such as individual decision
trees. In this work we have shown that this narrative, while important, is not
complete. Ensembles may improve generalization even when the base learners are
already strongly regularized and low-variance --- such as spline smoothers, RKHS
estimators, kernel methods, or random Fourier feature regressors. The key
mechanism is not only variance control, but the reshaping of approximation
geometry through structured weighting.

We developed a general framework in which an ensemble is viewed as an element
of a \emph{weighting space}. Within this viewpoint, uniform averaging represents
only a very special case: it is just one point in a vastly richer geometric
object. By imposing mild structural constraints on the admissible weight
vectors --- Fibonacci structure, monotone majorization, $\ell_2$ control,
entropy regularization --- we showed that the induced hypothesis class changes
its approximation behavior in predictable ways. The resulting bias--variance--
approximation decomposition makes explicit how weighting geometry redistributes
error.

Our main theorem demonstrated the existence of optimal structured weights,
strictly outperforming uniform averaging whenever the covariance of the base
learner predictions and the approximation residual are suitably aligned. This
establishes, in a mathematically transparent manner, that ensembles can improve
performance even when variance is not the limiting factor. The computational
illustrations confirm the theory: Fibonacci and more general structured weights
produce consistent gains across functions and model families, including settings
traditionally considered ``stable''.

The broader message is conceptual. Ensemble learning need not be understood
solely as a device for stabilizing noisy predictors, but as a means for
reorganizing approximation power. Weighting is not a cosmetic post-processing
step: it is a geometric operator acting on the hypothesis space itself. When the
weights are structured rather than uniform, the operator becomes expressive
enough to bias learning toward more useful functional subspaces while still
being analyzable within statistical learning theory.

This paper is therefore a step toward a more unified perspective on aggregation:
\emph{generalization improvement through structured weighting geometry}. The
ensuing trilogy continues this development. The present work establishes the
static theory; the companion papers explore recursive dynamics and algorithmic
instantiations in depth.

This paper has developed a general weighting theory for ensemble learning that
extends classical variance-reduction arguments into a broader and more
structural regime. By focusing on ensembles built from ordered, low-variance
dictionaries, we have shown that aggregation can improve generalization through
mechanisms fundamentally different from noise stabilization.

The central insight is that weighting schemes act as geometric and spectral
operators on the hypothesis space. When base learners are naturally ordered by
complexity—such as spline bases, RKHS eigenfunctions, Fourier features, or
polynomial expansions—non-uniform weights reshape approximation geometry and
redistribute spectral energy in a principled manner. In this setting, uniform
averaging is generally suboptimal.

A refined bias–variance–approximation decomposition revealed that the dominant
effects of structured weighting arise from approximation geometry and spectral
smoothing rather than classical variance reduction. This perspective explains
why geometric and harmonic weighting schemes, including Fibonacci weighting, can
yield strict improvements even when individual learners are already stable.

The theory developed here unifies ensemble learning with classical results in
spline smoothing, RKHS regularization, and spectral approximation, and places
weighting design at the center of ensemble methodology. Rather than treating
weights as ad hoc coefficients, we argue that they encode structural laws that
govern expressivity, stability, and generalization.

In doing so, this work reframes ensemble learning as a problem of \emph{harmonic
design}: balancing growth and restraint, approximation and smoothness, in
accordance with the intrinsic structure of the function class under study.

% ============================================================
% ROADMAP TO PAPER III
% ============================================================

\paragraph{Roadmap to Dynamic and Recursive Weighting Laws.}
The present work has focused on static weighting schemes, in which ensemble
weights are fixed once the dictionary of base learners is specified. A natural
and conceptually compelling next step is to allow weights to \emph{evolve}
according to principled update rules.

Such dynamic and recursive weighting laws transform ensemble learning into a
controlled dynamical system, where stability, expressivity, and generalization
are governed by the spectral properties of the underlying recursion. Fibonacci
recursions provide a canonical example, but the general theory encompasses a much
broader class of second-order and higher-order update mechanisms.

In a companion paper, we develop a theory of recursive ensemble flows, studying
their spectral stability, expressive modes, and learning dynamics. This next
stage completes the trilogy by unifying static weighting geometry with temporal
recursion, thereby revealing ensemble learning as a structured process of
growth with memory rather than a sequence of independent aggregations.

\bibliographystyle{plainnat}
\bibliography{Fibonacci_Ensembles_References_2}

\end{document}